\documentclass[10pt,twocolumn,letterpaper]{article}

\usepackage[letterpaper,top=0.72in,bottom=0.78in,left=0.72in,right=0.72in]{geometry}
\usepackage[T1]{fontenc}
\usepackage[utf8]{inputenc}
\usepackage{times}
\usepackage[hyphens]{url}
\usepackage{graphicx}
\usepackage{natbib}
\usepackage[font=small,labelfont=bf]{caption}
\usepackage{booktabs}
\usepackage{multirow}
\usepackage{array}
\usepackage[table]{xcolor}
\usepackage{amssymb}
\usepackage{mathtools}
\usepackage{amsthm}
\usepackage{thmtools}
\usepackage{algorithm}
\usepackage{algpseudocode}
\usepackage{microtype}
\usepackage[hidelinks,hypertexnames=false]{hyperref}
\usepackage[capitalize,nameinlink,noabbrev]{cleveref}


\usepackage{amsmath,amsfonts,bm}

\def\eqref#1{equation~\ref{#1}}

\def\1{\bm{1}}

\def\eps{{\epsilon}}

\DeclareMathAlphabet{\mathsfit}{\encodingdefault}{\sfdefault}{m}{sl}
\SetMathAlphabet{\mathsfit}{bold}{\encodingdefault}{\sfdefault}{bx}{n}

\newcommand{\E}{\mathbb{E}}

\newcommand{\R}{\mathbb{R}}

\newcommand{\softmax}{\mathrm{softmax}}

\newcommand{\Var}{\mathrm{Var}}

\newcommand{\Cov}{\mathrm{Cov}}

\DeclareMathOperator*{\argmax}{arg\,max}

\newcommand{\vobs}{\bm{v}}                 
\newcommand{\zspan}{\bm{s}}                
\newcommand{\ek}{\bm{e}}                   
\newcommand{\enull}{\bm{e}_{\varnothing}}  
\newcommand{\epsth}{\eps_\theta}           
\newcommand{\abar}{\bar{\alpha}}
\newcommand{\score}{\mathrm{score}}
\newcommand{\err}{\mathrm{err}}
\newcommand{\zvspan}{\bm{z}^{\mathrm{v}}}  
\newcommand{\SNR}{\mathrm{SNR}}
\newcommand{\Ldiff}{\mathcal{L}_{\mathrm{diff}}}
\newcommand{\Lico}{\mathcal{L}_{\mathrm{ico}}}
\newcommand{\Lclf}{\mathcal{L}_{\mathrm{clf}}}
\newcommand{\Ltask}{\mathcal{L}_{\mathrm{task}}}

\declaretheorem[name=Proposition]{proposition}
\declaretheorem[name=Assumption,style=definition]{assumption}
\declaretheorem[name=Lemma]{lemma}
\declaretheorem[name=Remark,style=remark]{remark}

\crefname{figure}{Figure}{Figures}
\Crefname{figure}{Figure}{Figures}
\crefname{table}{Table}{Tables}
\Crefname{table}{Table}{Tables}
\crefname{section}{Section}{Sections}
\Crefname{section}{Section}{Sections}
\crefname{equation}{Equation}{Equations}
\Crefname{equation}{Equation}{Equations}
\crefname{definition}{Definition}{Definitions}
\Crefname{definition}{Definition}{Definitions}
\crefname{proposition}{Proposition}{Propositions}
\Crefname{proposition}{Proposition}{Propositions}

\title{DiffImaginE: Imagine to Verify Entity Types with Diffusion}
\author{%
  \begin{minipage}{0.96\textwidth}
    \centering
    \normalsize
    Feng Zhang\textsuperscript{1,*},
    Feiyu Han\textsuperscript{2,*},
    Rongxin Yang\textsuperscript{6,3,*},
    Yang Liu\textsuperscript{3},
    Yancheng Chen\textsuperscript{2}\\[0.15em]
    Rui Wang\textsuperscript{4},
    Yingguang Yang\textsuperscript{5},
    Tian Xueyun\textsuperscript{2},
    Chongyang Zhang\textsuperscript{6},
    Hao Zheng\textsuperscript{6}\\[0.15em]
    Xu Kefu\textsuperscript{3},
    Congjing Ran\textsuperscript{7},
    Fuhai Chen\textsuperscript{1},
    Bin Chong\textsuperscript{3,\textdagger}\\[0.55em]
    \footnotesize
    \textsuperscript{1}Fuzhou University; 
    \textsuperscript{2}Chinese Academy of Sciences; 
    \textsuperscript{3}Peking University; 
    \textsuperscript{4}Alibaba Group\\
    \textsuperscript{5}University of Science and Technology of China; 
    \textsuperscript{6}Fullive Innovation (Beijing) AI Technology Co., Ltd.; 
    \textsuperscript{7}Wuhan University\\[0.35em]
    \textsuperscript{*}Equal contribution.\quad
    \textsuperscript{\textdagger}Corresponding author:
    \href{mailto:chongbin@pku.edu.cn}{chongbin@pku.edu.cn}
  \end{minipage}
}

\date{}

\begin{document}

\maketitle

\begin{abstract}
Multimodal named entity recognition (MNER) determines, for each candidate span, whether an entity-type hypothesis is supported by the joint textual and visual evidence. Existing imagine-and-compare verifiers typically map each $(\text{span},\text{type})$ pair to a single predicted visual feature and compare it with the observed image representation. Such deterministic imagination compresses the diverse visual realisations of an entity type into one prototype, making the verifier brittle when the same type appears through substantially different visual cues. Moreover, the resulting compatibility score lacks a probabilistic interpretation and provides only indirect supervision for rejecting confusable type hypotheses.

We propose \textbf{DiffImaginE}, which formulates MNER type verification as conditional latent diffusion inference. Given span-localised visual evidence, a type-conditioned denoiser predicts the noise injected into its standardised latent representation. The resulting denoising error provides an ELBO-consistent surrogate for the type-conditional negative log-likelihood, enabling different type hypotheses to be compared according to how well they explain the observed evidence. DiffImaginE preserves a standard multimodal encoder stack and replaces only the deterministic verifier with a classifier-free-guided diffusion scorer trained using Min-SNR weighting. To bridge generative likelihood estimation and discriminative prediction, we directly supervise per-type diffusion scores as classification logits, learn to aggregate evidence across noise levels, and use antithetic sampling to reduce Monte-Carlo comparison variance. Our theoretical analysis shows that classifier-free guidance sharpens the induced type posterior and characterises when antithetic pairing yields lower variance at equal denoiser cost. Experiments on Twitter-2015 and Twitter-2017 show that DiffImaginE consistently improves over a matched deterministic ImaginE control under the same encoder, auxiliary objectives, and evaluation protocol, with further support from ablations and paired significance tests.
\end{abstract}

\section{Introduction}
\label{sec:intro}

Named entity recognition (NER) on social-media text is challenging because posts are often short, noisy, and lexically ambiguous. Multimodal NER (MNER) uses accompanying images to resolve such ambiguity. For example, ``Jordan dropped 40'' may refer to a person, a brand, or another entity, while the associated image can reveal whether ``Jordan'' denotes an athlete, an organisation, or a miscellaneous entity~\citep{moon2018multimodal,zhang2018adaptive,lu2018visual}. A widely adopted formulation therefore enumerates candidate spans and verifies whether each proposed entity type is compatible with the joint text--image context~\citep{yu2020umt,chen2022hvpnet,wang2022ita}. Under this formulation, verification is central: the encoders provide contextualised multimodal representations, but the final decision depends on the compatibility assigned to each $(\text{span},\text{type})$ hypothesis.

Prior work commonly performs verification through \emph{imagination}. Given a span and candidate type, a predictor synthesises the visual feature expected under that hypothesis, and a comparator measures its agreement with the observed evidence~\citep{chen2022hvpnet,xu2022maf}. Although intuitive, this deterministic formulation compresses the diverse visual realisations of a type into a single imagined point. A \texttt{PER} mention may appear as a frontal face, profile, distant figure, or action scene, while an \texttt{ORG} mention may be supported by a logo, storefront, uniform, or advertisement. A single prototype cannot adequately capture such intra-type variation, particularly for heterogeneous and visually overlapping categories such as \texttt{MISC}.

Moreover, deterministic compatibility scores provide no explicit likelihood interpretation. They measure resemblance to one imagined feature, but not how plausibly the observed evidence can be explained under each competing type. Negative hypotheses therefore receive mainly indirect contrastive pressure, allowing related types to obtain similarly high scores. A more suitable verifier should instead ask: \emph{under which type hypothesis is the observed span-localised visual evidence most plausible?}

Diffusion models offer a natural mechanism for this comparison. A conditional denoiser trained across multiple corruption levels models a distribution of possible observations rather than a single representative point. Its expected class-conditional denoising error can be interpreted as a negative variational bound on the conditional log-likelihood, enabling classification by comparing reconstruction errors across hypotheses~\citep{li2023diffusionclassifier,clark2023texttoimage}. However, existing diffusion classifiers mainly operate on full images. Applying them to MNER is non-trivial because the scored object is a low-dimensional, span-conditioned cross-modal latent whose scale changes with the jointly trained encoder, while the generative score must ultimately support discriminative span classification and rejection of the non-entity hypothesis.

We propose \textbf{DiffImaginE}, a conditional latent diffusion verifier for MNER. It preserves the standard text encoder, vision encoder, and span--visual interaction modules, while replacing deterministic imagination with type-conditioned diffusion scoring. For each candidate span, cross-attention extracts span-localised visual evidence, which is standardised and corrupted with Gaussian noise. A shared denoiser predicts the injected noise under each candidate type, and the resulting Min-SNR-weighted errors define type-specific verification scores. A NULL-conditioned branch further enables classifier-free guidance~\citep{ho2022cfg,hang2023minsnr}. DiffImaginE thus evaluates each hypothesis according to how well its conditional denoising process explains the same observation.

Because a generative likelihood surrogate is not necessarily an optimal discriminative classifier, we further adapt the scorer in three ways. We directly supervise per-type scores as classification logits, learn timestep aggregation weights to emphasise discriminative noise levels, and use antithetic noise pairs to reduce Monte-Carlo comparison variance when the odd component dominates. The resulting scores are fused with the original multimodal representations by the final entity classifier.

DiffImaginE changes only the verification mechanism. Our matched \textbf{ImaginE} control uses the same encoders, span-localisation modules, auxiliary objectives, classifier head, optimisation procedure, and decoding protocol, but retains deterministic imagine-and-compare verification. This controlled comparison isolates the effect of replacing single-point imagination with distributional diffusion scoring (\cref{sec:baselines}).

Our contributions are summarised as follows:
\begin{itemize}
    \item \textbf{Diffusion-based type verification.}
    We formulate MNER verification as conditional latent diffusion inference, replacing single-point visual imagination with an ELBO-based distributional scorer.

    \item \textbf{Adaptation to multimodal span latents.}
    We combine latent standardisation, span- and type-conditioned denoising, classifier-free guidance, and Min-SNR weighting for low-dimensional span-localised evidence.

    \item \textbf{Discriminative and variance-reduced scoring.}
    We supervise diffusion scores as type logits, learn their aggregation across timesteps, and derive the condition under which antithetic estimation reduces comparison variance (\cref{sec:disc,sec:theory}).

    \item \textbf{Controlled evaluation.}
    Experiments on Twitter-2015 and Twitter-2017 show consistent improvements over a matched deterministic verifier, supported by ablations and paired tests.
\end{itemize}

\begin{figure}[t]
    \centering
    \includegraphics[width=\linewidth]{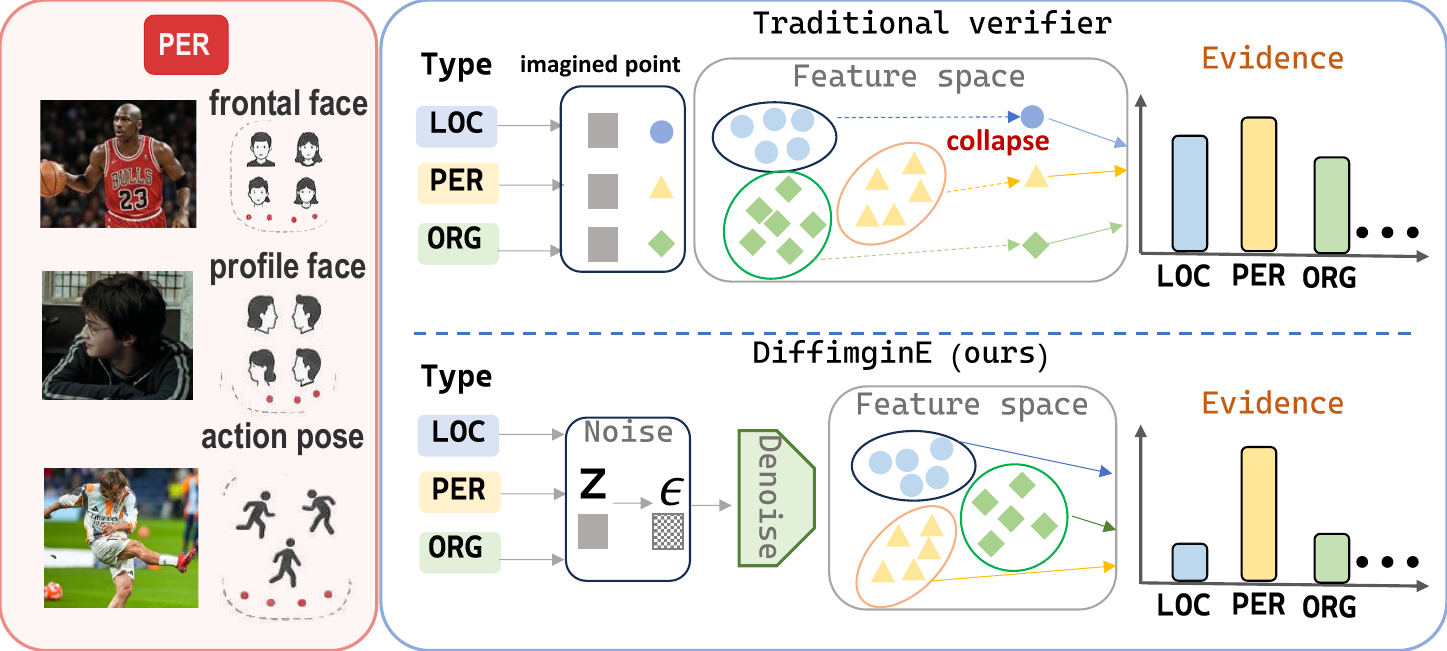}
    \caption{DiffImaginE replaces single-point imagine-and-compare verification with type-conditioned denoising on span-localised visual evidence: each type is scored by its denoising error, which better separates confusable hypotheses when visual evidence for a type is diverse.}
    \label{fig:motivation}

\end{figure}

\section{Preliminaries}
\label{sec:prelim}

We introduce the span-level MNER formulation, the visual evidence scored by DiffImaginE, and the diffusion-classifier estimator underlying our verifier.

\subsection{Multimodal NER and Span-Level Type Verification}
\label{sec:mner}

An MNER instance consists of a token sequence $\bm{x}=(x_1,\dots,x_L)$ and an image $\bm{m}$. The task is to identify all entity spans and assign each a type from
$\mathcal{Y}=\{\texttt{PER},\texttt{LOC},\texttt{ORG},\texttt{MISC}\}$.
Following the enumerate-and-verify formulation, we enumerate candidate spans
$s=(i,j)$ satisfying $1\leq i\leq j\leq L$ and $j-i+1\leq W$, and classify each into one of
$K=|\mathcal{Y}|+1=5$ labels, including the non-entity class \texttt{O}. Thus, the verifier jointly determines whether a candidate is an entity and, if so, its type.

\subsection{Span-Localised Visual Evidence}
\label{sec:evidence}

A text encoder (RoBERTa~\citep{liu2019roberta}) produces contextual token states, while a vision encoder (CLIP-ViT~\citep{radford2021clip,dosovitskiy2021vit}) produces global and patch-level image features. Both are projected into a shared $d$-dimensional space. Boundary-aware pooling yields a span representation $\zspan\in\R^d$, which attends over the image patches to produce span-localised visual evidence
$\zvspan_s\in\R^d$~\citep{vaswani2017attention,chen2022hvpnet}. DiffImaginE scores type hypotheses against $\zvspan_s$; all upstream encoders and interaction modules are shared with the deterministic baseline.

\subsection{Diffusion Models and the Diffusion Classifier}
\label{sec:diffusion}

We use a variance-preserving DDPM~\citep{ho2020ddpm,nichol2021improved} with $T$ steps and cosine schedule $\{\abar_t\}_{t=0}^{T-1}$. Given a clean latent $\vobs\in\R^d$, the forward process is
\begin{equation}
\vobs_t
=
\sqrt{\abar_t}\,\vobs
+
\sqrt{1-\abar_t}\,\eps,
\qquad
\eps\sim\mathcal{N}(0,I),
\label{eq:qsample}
\end{equation}
where a conditional denoiser $\epsth(\vobs_t,\bm{c},t)$ predicts the injected noise under condition $\bm{c}$. The corresponding signal-to-noise ratio is
$\SNR(t)=\abar_t/(1-\abar_t)$.

A diffusion classifier~\citep{li2023diffusionclassifier,clark2023texttoimage} evaluates class hypothesis $k$ through its expected denoising error:
\begin{equation}
\err_k
=
\E_{t,\eps}\!\left[
w(t)
\left\|
\eps-
\epsth\!\left(
\sqrt{\abar_t}\vobs+
\sqrt{1-\abar_t}\eps,
\bm{c}_k,t
\right)
\right\|_2^2
\right].
\label{eq:err}
\end{equation}
With an appropriate $w(t)$, $\err_k$ corresponds to the negative ELBO of
$\log p_\theta(\vobs\mid\bm{c}_k)$ up to a class-independent constant, so lower error indicates a more likely class~\citep{kingma2021vdm,ho2020ddpm}.

DiffImaginE sets $\vobs=\zvspan_s$ and uses the composite condition
$\bm{c}_k=(\zspan,\ek_k)$, yielding the denoiser
$\epsth(\vobs_t,\zspan,\ek_k,t)$. The resulting error measures how well type $k$ explains the visual evidence associated with span $s$. We estimate the expectation by Monte-Carlo sampling over $(t,\eps)$, writing $\err_{k,n}$ for an individual estimate and $\err_k$ for the aggregate.

\section{The DiffImaginE Model}
\label{sec:method}

\subsection{Overview}
\label{sec:overview}
DiffImaginE replaces the deterministic imagine-and-compare verifier with the conditional diffusion scorer of \cref{eq:err}. For each valid span, span--visual cross-attention produces $\zvspan_s$, which is standardised (\cref{sec:latentnorm}), scored under all $K$ type hypotheses (\cref{sec:scorer}), and fused with textual and visual representations for final classification (\cref{fig:overview}). Reverse imagination and the visual-relevance gate are inherited unchanged from ImaginE and detailed in the technical appendix. Padding and invalid spans are masked from all diffusion forwards.

\begin{figure}[t]
    \centering
    \includegraphics[width=\linewidth]{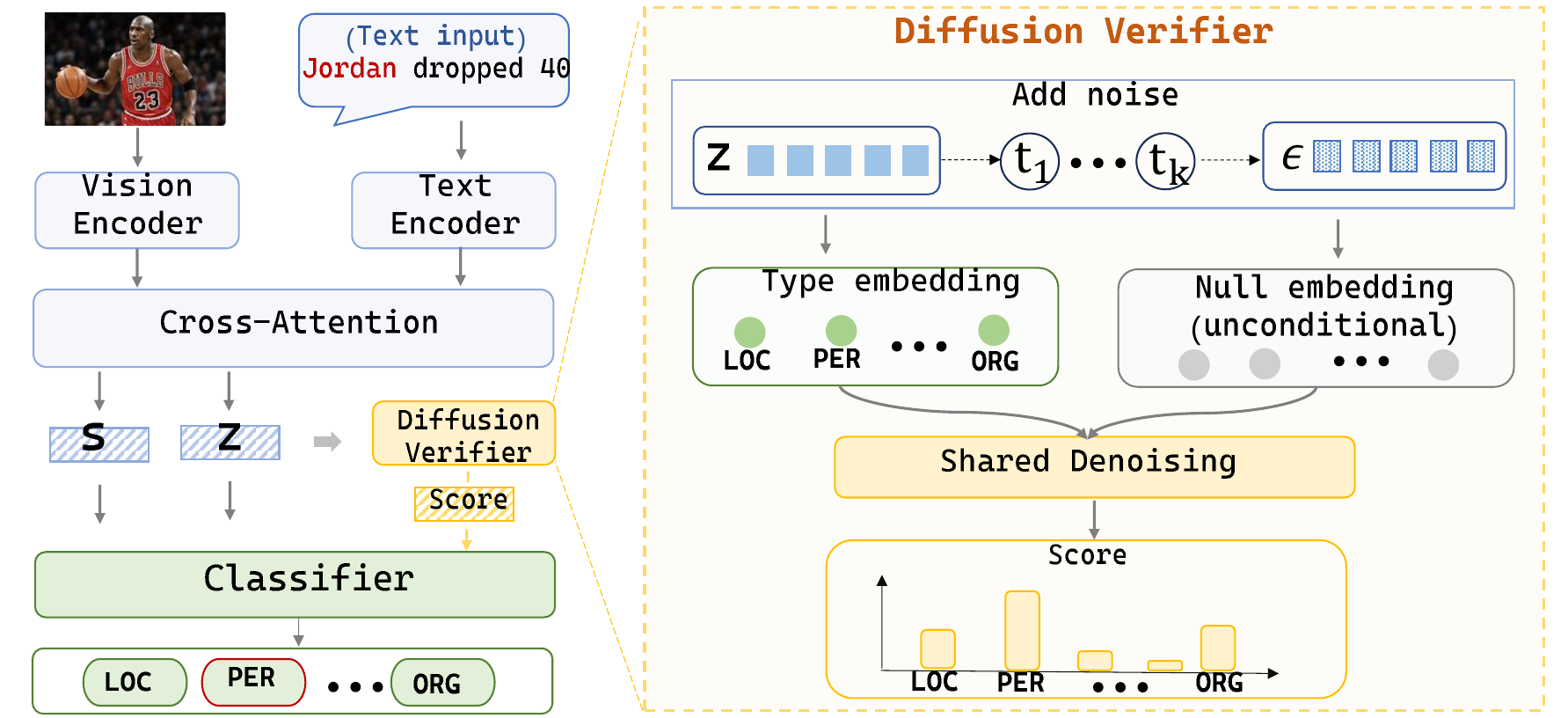}
    \caption{DiffImaginE pipeline: cross-attention yields span representation $s$ and span-localised visual evidence $z$; the diffusion verifier scores each type by type-conditioned denoising (with a NULL branch for classifier-free guidance), and the classifier fuses these scores with multimodal features for the final prediction.}
    \label{fig:overview}
    \vspace{-1.5em}
\end{figure}

\subsection{Latent Standardisation}
\label{sec:latentnorm}
The signal scale of $\zvspan_{s}$ is set by the encoder, not by the diffusion schedule, so a cosine schedule calibrated for unit-variance data is mismatched on a raw, low-dimensional cross-modal latent. Before diffusion we estimate a per-dimension mean $\bm{\mu}$ and standard deviation $\bm{\sigma}$ of $\zvspan_{s}$ over a calibration set of training batches and operate the diffusion process on the standardised latent $\vobs = (\zvspan_{s} - \bm{\mu}) \oslash \bm{\sigma}$, mirroring the latent scaling of latent diffusion models~\citep{rombach2022ldm}. Because the encoder drifts during joint training, $(\bm{\mu}, \bm{\sigma})$ are re-estimated after the denoiser warmup and periodically thereafter, so the SNR curve tracks the live encoder. Unlike image diffusion, the scored object here is a $d$-dimensional span summary rather than a spatial field, so miscalibrated scale directly distorts the relative noise levels at which types are compared. We treat standardisation as a prerequisite for stable scoring and isolate its effect in the \texttt{no\_latent\_norm} ablation.

\subsection{Type-Conditioned Denoiser}
\label{sec:denoiser}
The denoiser $\epsth(\vobs_t, \zspan, \ek_k, t)$ is an adaptive-LayerNorm (AdaLN) Transformer/MLP block~\citep{peebles2023dit}: the timestep embedding and the type embedding $\ek_k$ are mapped to per-layer scale/shift parameters that modulate the normalised activations, while the span representation $\zspan$ enters as an input feature. The type-embedding table has $K + 1$ rows: $K$ rows for the supervised classes (including \texttt{O}) and one extra NULL row $\enull$, the unconditional embedding used for classifier-free guidance (\cref{sec:scorer}). The NULL row is a guidance device, distinct from the \texttt{O} class, which is a genuine type hypothesis the scorer must rank. The AdaLN projection is half-zero-initialised so the type condition has a small but non-zero effect from the first step, avoiding a long regime in which the conditioning signal is silenced.

\subsection{Classifier-Free-Guided ELBO Scoring}
\label{sec:scorer}
For span $s$ with standardised evidence $\vobs$, the per-type error of \cref{eq:err} is estimated by Monte-Carlo over timesteps and noise. Following classifier-free guidance~\citep{ho2022cfg,dhariwal2021classifier}, we also evaluate the denoiser under the NULL condition $\enull$ and form a guided score with guidance scale $g \ge 0$:
\begin{equation}
\err_{\varnothing} = \E_{t,\eps}\!\bigl[ w(t)\,\|\eps - \epsth(\vobs_t, \zspan, \enull, t)\|_2^2\bigr],
\label{eq:errnull}
\end{equation}
\begin{equation}
\begin{aligned}
\score_k^{(g)} &= -\bigl(\err_k - g\,(\err_{\varnothing} - \err_k)\bigr)\\
&= -(1 + g)\,\err_k + g\,\err_{\varnothing}.
\end{aligned}
\label{eq:score}
\end{equation}
The same noise sample $\eps$ is shared across all $K{+}1$ hypotheses at each $(t,s)$, so comparisons differ only through the denoiser response. Training draws timesteps by stratified sampling over $M$ strata of $[0,T)$; evaluation stacks a fixed set of $N$ timesteps in $[t_{\mathrm{lo}},t_{\mathrm{hi}}]$ into one batched forward. The resulting errors $\err^{(g)}_{k,n}=(1+g)\err_{k,n}-g\err_{\varnothing,n}$ are aggregated as $\score_k=-\sum_n w(t_n/T)\err^{(g)}_{k,n}$. Because $\err_{\varnothing}$ is type-independent, $g$ acts as a posterior-sharpening factor in the idealised ELBO regime (\cref{prop:elbo}).

\subsection{Min-SNR Weighting and the Denoising Objective}
\label{sec:minsnr}
The denoiser is trained with a denoising score-matching loss weighted per timestep by the Min-SNR-$\gamma$ rule~\citep{hang2023minsnr}, which equalises the gradient contribution across noise levels by clipping the SNR weight at $\gamma$:
\begin{equation}
\begin{aligned}
\Ldiff &= \E_{t, \eps}\!\bigl[\, w_\gamma(t)\,\|\eps - \epsth(\vobs_t, \zspan, \ek_{y}, t)\|_2^2\,\bigr],\\
w_\gamma(t) &= \frac{\min(\SNR(t), \gamma)}{\SNR(t)},
\end{aligned}
\label{eq:ldiff}
\end{equation}
where $y$ is the gold type of the span (\texttt{O} is a valid supervised type, so the score on \texttt{O} is well-defined at test time). With probability $p_{\mathrm{cf}}$ the gold condition $\ek_{y}$ is replaced by $\enull$ during training, which fits the unconditional denoiser branch that classifier-free guidance contrasts against in \cref{eq:score}. The classifier scores and $\Ldiff$ are produced by a single fused denoiser forward per Monte-Carlo step.

\subsection{Training Objectives}
\label{sec:objectives}
The verifier and encoder are trained jointly. The primary term is class-weighted cross-entropy $\Ltask$, with a lower weight for \texttt{O} to mitigate span-enumeration imbalance. The diffusion path adds $\Ldiff$ (\cref{eq:ldiff}) and a score-level contrastive loss $\Lico$ over entity spans within the evaluation timestep window. Reverse-imagination alignment, VICReg regularisation~\citep{bardes2022vicreg}, BIO supervision, cross-modal InfoNCE~\citep{oord2018infonce,radford2021clip}, grounding, type ranking, and hard-example objectives are inherited by both DiffImaginE and the deterministic control. Their definitions and weights are deferred to the technical appendix, ensuring that the main comparison isolates the verifier.

\subsection{From a Generative Score to a Discriminative Scorer}
\label{sec:disc}
A bare ELBO score is a generative quantity used as a discriminative feature, which is known to be suboptimal as a classifier. DiffImaginE closes this gap with three independent components.
\paragraph{Likelihood-based classification objective.} We supervise the per-type scores \emph{directly}: treating $\score_k$ as a logit, $\Lclf$ is a cross-entropy over all $K$ types (including \texttt{O}) with a learnable temperature $\tau_{\mathrm{clf}}$,
\begin{equation}
\Lclf = \mathrm{CE}\!\bigl(\softmax_k(\score_k / \tau_{\mathrm{clf}}),\, y\bigr),
\label{eq:lclf}
\end{equation}
which aligns the ordering of the scores with the gold type during training, so the scores are trained as classification logits rather than read off a frozen generative model~\citep{xian2024graphclf}. We parameterise $\tau_{\mathrm{clf}} = \exp(\theta_{\tau})$ to keep it positive, initialised so that $\tau_{\mathrm{clf}} = 0.1$.
\paragraph{Learnable timestep aggregation.} The score-side weighting $w(t)$ of \cref{eq:err} is left unspecified by the ELBO up to the variational choice; we instantiate it as a learned aggregation. Instead of a uniform mean over the active timesteps, a tiny zero-initialised network maps the normalised timestep $t/T$ to a logit, and the per-timestep errors are aggregated with the resulting softmax weights; zero initialisation makes the aggregation identical to a uniform mean at start-up, after which the model up-weights the most discriminative noise levels~\citep{jeong2025compositionality}. This weight is distinct from the Min-SNR loss weight $w_\gamma(t)$ of \cref{eq:ldiff}: the former shapes the evaluation-time score, while $w_\gamma(t)$ acts only on $\Ldiff$.
\paragraph{Antithetic variance reduction.} The Monte-Carlo noise is drawn in same-timestep pairs $(\eps, -\eps)$. Because the type-discriminative part of the score is, to leading order, odd in $\eps$ and shared across types, antithetic pairing lowers the variance of the per-timestep type comparison relative to an i.i.d.\ estimator of equal cost when the odd error component dominates~\citep{jia2026antithetic}; \cref{prop:antithetic} makes the condition precise. Finally, a score LayerNorm with a learnable temperature rescales the $K$ score channels so they are not dwarfed by the $d$-dimensional representations when fused in the classifier.

\subsection{Decoding}
\label{sec:decoding}
At test time, weighted-interval-scheduling dynamic programming selects the maximum-confidence non-overlapping span set in $O(S\log S)$ time, where $S=O(LW)$. A span is accepted only when its best entity logit exceeds the \texttt{O} logit and its diffusion margin $\max_{k\neq\texttt{O}}\score_k-\score_{\texttt{O}}$ exceeds an abstention threshold. The guidance scale and both margins are tuned on the development set and reused unchanged for strict test $F_1$; greedy decoding is retained only as an ablation.

\subsection{Theoretical Analysis}
\label{sec:theory}
We state two results about the scorer; both are proved in the technical appendix. The first relates guided scoring to a tempered posterior under an idealised uniform type prior; in training, the empirical dominance of \texttt{O} is handled by class weighting in $\Ltask$ and the classifier head rather than inside the score layer. The second compares antithetic and i.i.d.\ estimators of the type-difference score.

\begin{assumption}[ELBO-consistent weighting]
\label{ass:elbo}
We adopt the standard diffusion-classifier modelling assumption~\citep{ho2020ddpm,kingma2021vdm,li2023diffusionclassifier}: the timestep weighting $w(t)$ in \cref{eq:err} is the variational weighting under which $\err_k$ equals the negative variational bound on $\log p_\theta(\vobs \mid \zspan, \ek_k)$, and the gap between this bound and the exact log-likelihood, together with the prior and reconstruction terms and the entropy of the forward process, is folded into a constant $C$ independent of the type $k$. We thus treat $\err_k = -\log p_\theta(\vobs \mid \zspan, \ek_k) + C$ as the operative model rather than an exact identity.
\end{assumption}

\begin{proposition}[Guided score is a tempered posterior]
\label{prop:elbo}
Under \cref{ass:elbo} and a uniform prior over the $K$ types, the unguided score $\score_k = -\err_k$ satisfies $\softmax_k(\score_k) = p_\theta(\ek_k \mid \vobs, \zspan)$, so $\argmax_k \score_k$ is the Bayes-optimal type and $\softmax_k(\score_k/\tau)$ is a temperature-$\tau$ posterior estimate. For guided scoring, with any guidance scale $g \ge 0$ and temperature $\tau > 0$, the classifier-free-guided score of \cref{eq:score} satisfies
\begin{equation}
\begin{aligned}
\softmax_k\!\bigl(\score_k^{(g)} / \tau\bigr) &= \softmax_k\!\bigl(\score_k / \tfrac{\tau}{1+g}\bigr)\\
&\propto p_\theta(\ek_k \mid \vobs, \zspan)^{(1+g)/\tau}.
\end{aligned}
\label{eq:cfg-temper}
\end{equation}
That is, guidance and temperature combine into a single effective exponent $(1+g)/\tau$ on the posterior: increasing $g$ (or lowering $\tau$) sharpens the type distribution, while decreasing $g$ (or raising $\tau$) flattens it, with no effect on the $\argmax$. At $\tau = 1$ the guided score is exactly the posterior raised to the power $1 + g$.
\end{proposition}

The proof is the diffusion-classifier ELBO identity followed by the observation that $\err_{\varnothing}$ is type-independent and hence cancels inside the softmax, leaving an exact rescaling of the logits by $1 + g$. This reading motivates the dev-tuning grid for $g$ and the choice of $-\err_k$ as the initial logit parameterisation in \cref{eq:lclf}. We next evaluate whether the resulting scorer improves MNER in practice.

\begin{assumption}[Square-integrable, symmetric noise]
\label{ass:noise}
The noise $\eps \sim \mathcal{N}(0, I)$ is symmetric ($\eps \stackrel{d}{=} -\eps$), and for a fixed timestep $t$ and span $s$ the per-sample squared errors $\phi_k(\eps) = \|\epsth(\sqrt{\abar_t}\vobs + \sqrt{1 - \abar_t}\,\eps, \zspan, \ek_k, t) - \eps\|_2^2$ are square-integrable for every $k$.
\end{assumption}

\begin{proposition}[Antithetic variance reduction]
\label{prop:antithetic}
Fix $t$ and a type pair $(k, j)$, and let $f(\eps) = \phi_k(\eps) - \phi_j(\eps)$ be the single-sample type-difference, with even/odd parts $f_e(\eps) = \tfrac12(f(\eps) + f(-\eps))$ and $f_o(\eps) = \tfrac12(f(\eps) - f(-\eps))$. Under \cref{ass:noise}, the antithetic estimator $\widehat{D}^{\,\mathrm{anti}} = \tfrac12(f(\eps) + f(-\eps))$ and the two-sample i.i.d.\ estimator $\widehat{D}^{\,\mathrm{iid}} = \tfrac12(f(\eps^{(1)}) + f(\eps^{(2)}))$ use the same number of denoiser evaluations and are both unbiased for $\E[f] = \err_k - \err_j$, with
\begin{equation}
\begin{aligned}
\Var\bigl(\widehat{D}^{\,\mathrm{anti}}\bigr) &= \Var(f_e),\\
\Var\bigl(\widehat{D}^{\,\mathrm{iid}}\bigr) &= \tfrac12\bigl(\Var(f_e) + \Var(f_o)\bigr).
\end{aligned}
\label{eq:anti-var}
\end{equation}
Hence $\Var(\widehat{D}^{\,\mathrm{anti}}) \le \Var(\widehat{D}^{\,\mathrm{iid}})$ if and only if $\Var(f_e) \le \Var(f_o)$: the antithetic estimator achieves lower variance precisely when the odd component carries at least as much variance as the even one. A first-order expansion of the denoiser about $\eps = 0$ shows the odd part of $f$ is dominated by $-2\langle \eps,\, \bm{a}_k - \bm{a}_j\rangle$, with $\bm{a}_k = \epsth(\sqrt{\abar_t}\vobs, \zspan, \ek_k, t)$. The criterion therefore holds when the inter-type gap $\bm{a}_k - \bm{a}_j$ is large relative to the denoiser's local sensitivity to $\eps$. Proofs and the termwise aggregation argument are in the technical appendix.
\end{proposition}

\section{Experiments}
\label{sec:experiments}

Having specified the scorer (\cref{sec:method}) and stated its theoretical properties (\cref{sec:theory}), we ask whether diffusion-based type verification improves MNER over deterministic imagination and whether the ablations match those properties. We evaluate on Twitter-2015 and Twitter-2017; the protocol is summarised in \cref{sec:setup}.

\subsection{Datasets}
\label{sec:datasets}
We evaluate on the two standard Twitter-MNER benchmarks (\cref{tab:datasets}). The text splits follow the canonical UMT release~\citep{yu2020umt}; the accompanying images follow the HVPNeT distribution~\citep{chen2022hvpnet}. The data-preparation pipeline pins the upstream mirror to a fixed commit and verifies a golden SHA-256 of every split file, so any silent drift raises an error rather than changing the numbers.

\begin{table}[ht]
\caption{Datasets used in the protocol. Entity types are \texttt{PER}/\texttt{LOC}/\texttt{ORG}/\texttt{MISC}.}
\label{tab:datasets}
\centering
\begin{tabular}{lrrrr}
\toprule
\textbf{Dataset} & \textbf{Train} & \textbf{Dev} & \textbf{Test} & \textbf{Types} \\
\midrule
Twitter-2015 & 4{,}000 & 1{,}000 & 3{,}257 & 4 \\
Twitter-2017 & 3{,}373 & 723   & 723   & 4 \\
\bottomrule
\end{tabular}
\vspace{-1.5em}
\end{table}

\subsection{Main Results}
\label{sec:main-results}
\begin{table*}[!t]
\vspace{1.5em}
\caption{Main results: span-level strict precision (P), recall (R), and $F_1$ (\%) on the two Twitter-MNER benchmarks. Published baselines are cited from their original papers; a dash marks a value not reported in the source. ImaginE and DiffImaginE (below the rule) are, respectively, our matched deterministic control and our full model, which differ only in the verifier. DiffImaginE (last row) is shown in bold; the best value in each column is likewise in bold.}
\label{tab:main}
\centering
\small
\setlength{\tabcolsep}{5pt}
\begin{tabular}{lcccccc}
\toprule
& \multicolumn{3}{c}{\textbf{Twitter-2015}} & \multicolumn{3}{c}{\textbf{Twitter-2017}} \\
\cmidrule(lr){2-4} \cmidrule(lr){5-7}
\textbf{Method} & \textbf{P} & \textbf{R} & \textbf{F1} & \textbf{P} & \textbf{R} & \textbf{F1} \\
\midrule
UMT~\citep{yu2020umt}                  & 71.67 & 75.23 & 73.41 & 85.28 & 85.34 & 85.31 \\
Co-attention~\citep{zhang2018adaptive} & 69.87 & 74.59 & 72.15 & 85.13 & 83.20 & 84.10 \\
HVPNeT~\citep{chen2022hvpnet}          & 73.87 & 76.82 & 75.32 & 85.84 & 87.93 & 86.87 \\
MAF~\citep{xu2022maf}                  & 71.86 & 75.10 & 73.42 & 86.13 & 86.38 & 86.25 \\
De-Bias~\citep{zhang2023reducing}      & 74.45 & 76.13 & 75.28 & 87.59 & 86.11 & 86.84 \\
SMNER~\citep{zhou2022span}             & 75.34 & 76.81 & 76.06 & \textbf{88.15} & 87.47 & 87.81 \\
VEC-MNER~\citep{wei2024vec}            & 74.56 & 75.23 & 74.89 & 87.42 & 87.61 & 87.51 \\
SEPA~\citep{ding2025sepa}              & 74.71 & 76.84 & 75.76 & 87.50 & 87.05 & 87.27 \\
\midrule
ImaginE (ours, deterministic)          & 76.32 & 74.58 & 75.44 & 86.60 & 88.86 & 87.72 \\
\textbf{DiffImaginE (ours)}            & \textbf{77.41} & \textbf{76.93} & \textbf{77.17} & \textbf{87.72} & \textbf{89.16} & \textbf{88.44} \\
\bottomrule
\end{tabular}
\vspace{-1em}
\end{table*}
\Cref{tab:main} reports strict precision/recall/$F_1$ on both benchmarks for published baselines, the matched ImaginE control, and DiffImaginE. The primary contrast is DiffImaginE vs.\ ImaginE (\cref{sec:baselines}); per-type $F_1$ and a paired test appear in \cref{tab:per-type}.

On Twitter-2015, DiffImaginE improves over ImaginE by $+1.73$ strict $F_1$ ($77.17$ vs.\ $75.44$), driven mainly by higher precision ($77.41$ vs.\ $76.32$) and recall ($76.93$ vs.\ $74.58$). On Twitter-2017, swapping the ImaginE verifier for the diffusion scorer raises strict $F_1$ from $87.72$ to $88.44$ ($+0.72$), with precision rising from $86.60$ to $87.72$ and recall from $88.86$ to $89.16$. The paired per-type test in \cref{tab:per-type} rejects the null at $p=0.032$, so the gain is statistically significant at the $0.05$ level. Because DiffImaginE and ImaginE share the same encoder and training recipe (\cref{sec:baselines}), we attribute the improvement to the verifier. DiffImaginE also reaches the best listed $F_1$ on both datasets ($77.17$ on Twitter-2015 and $88.44$ on Twitter-2017); published baselines use different encoders and fusion designs and are shown for context only. The larger Twitter-2015 margin suggests that diffusion scoring helps most when visual evidence is noisier and single-point imagination is brittle.

\subsection{Per-Type Results}
\label{sec:per-type}
\Cref{tab:per-type} breaks down the Twitter-2017 contrast from \cref{sec:main-results}. The paired test compares DiffImaginE and ImaginE on the same training seeds; it rejects the null at $p=0.032$. Gains concentrate on \texttt{PER} ($92.76$ to $93.91$), where faces provide strong visual cues, and on \texttt{ORG} ($85.44$ to $86.24$), where logos and brand imagery help disambiguation. \texttt{LOC} is unchanged ($87.12$ vs.\ $87.08$), and \texttt{MISC} improves only slightly ($75.06$ to $75.84$), remaining the hardest type because of high visual and lexical diversity. The per-type pattern is consistent with the diffusion scorer exploiting diverse visual evidence rather than a single imagined feature vector.

\begin{table}[ht]
\caption{Per-type strict $F_1$ ($\uparrow$) on Twitter-2017 and the paired DiffImaginE-vs-ImaginE test. DiffImaginE (last row) is shown in bold; the best value per type is likewise in bold.}
\label{tab:per-type}
\centering
\resizebox{\columnwidth}{!}{%
\begin{tabular}{lccccc}
\toprule
\textbf{Method} & \textbf{PER} & \textbf{LOC} & \textbf{ORG} & \textbf{MISC} & \textbf{paired $p$} \\
\midrule
ImaginE (ours, deterministic) & 92.76 & \textbf{87.12} & 85.44 & 75.06 & 0.032 \\
\textbf{DiffImaginE (ours)}   & \textbf{93.91} & \textbf{87.08} & \textbf{86.24} & \textbf{75.84} & \textbf{0.032} \\
\bottomrule
\end{tabular}%
}
\vspace{-1em}
\end{table}

\subsection{Qualitative Case Study}
\label{sec:case-study}
\Cref{fig:case-study} visualises how DiffImaginE maps localised visual evidence to a type decision. The input establishes disambiguating comic context for the person-like mention ``Donald Duck.'' The span--visual attention (SVA) panel answers \emph{where} the model looks: the shared encoder grounds the span in character and comic regions, but attention does not assign a label. The prediction rows hold this evidence constant and isolate the verifier: ImaginE's single imagined-feature match returns \texttt{PER}, whereas DiffImaginE recovers \texttt{MISC}. The five-cell bar answers \emph{which type} the evidence supports. Each cell schematically gives the normalised contribution from one fixed noise level to the aggregated \texttt{MISC}-over-\texttt{PER} margin; darker blue means stronger positive support. Learned weighting combines these probes into a \texttt{MISC}-favouring signal that helps the fused classifier overcome the person-name prior. Together, the panels expose two advantages: multi-noise validation across complementary signal-to-noise regimes and likelihood-based comparison of how well competing type-conditioned denoisers explain the observed latent.

\begin{figure}[!t]
    \centering
    \includegraphics[width=0.84\columnwidth]{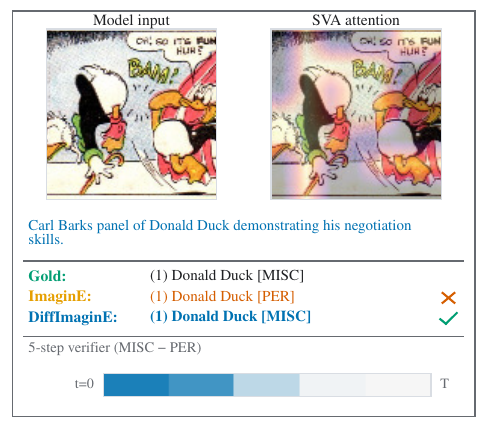}
    \caption{From evidence localisation to type verification. The input supplies disambiguating context, SVA shows where mention-relevant evidence is localised, and the prediction rows contrast single-point ImaginE with DiffImaginE. The five cells schematically show fixed-noise contributions to the \texttt{MISC}-over-\texttt{PER} margin; learned aggregation produces the final verification signal, with darker blue denoting stronger positive support. }
    \label{fig:case-study}
    
\end{figure}

\begin{table}[ht]
\caption{Ablation results on Twitter-2017 (strict $F_1$, $\uparrow$). Each variant is the mean over three seeds (42/43/44). Bold marks the reference configuration (\texttt{main}) and the highest $F_1$ among variants.}
\label{tab:ablation-results}
\centering
\setlength{\tabcolsep}{4pt}
\resizebox{\columnwidth}{!}{%
\begin{tabular}{lc@{\hspace{1.2em}}lc}
\toprule
\textbf{Variant} & \textbf{F1} & \textbf{Variant} & \textbf{F1} \\
\midrule
\textbf{main} & \textbf{88.78} & no\_clf\_diffusion & 88.68 \\
no\_latent\_norm & 88.58 & no\_learnable\_tw & 88.42 \\
g\_fixed\_1 & 88.40 & no\_warmup & 88.39 \\
no\_cfg & 88.38 & no\_l\_ico & 88.21 \\
greedy\_decode & 88.17 & no\_antithetic & 88.14 \\
no\_min\_snr & 88.10 & no\_score\_norm & 88.09 \\
no\_diffusion & 87.71 & & \\
\bottomrule
\end{tabular}%
}

\end{table}

Unless noted, ablation and budget-sweep rows are means over three seeds, whereas \cref{tab:main} reports the single selected operating point (EMA weights at the best-dev threshold); the two are therefore not directly comparable digit-for-digit, which also explains why \texttt{no\_diffusion} ($87.71$) and the ImaginE control in \cref{tab:main} ($87.72$) differ by $0.01$. Removing the diffusion verifier entirely (\texttt{no\_diffusion}) causes the largest drop, from $88.78$ to $87.71$ ($-1.07$ $F_1$), confirming that the gain comes from the scorer rather than the shared stack. Score normalisation ($88.09$), Min-SNR weighting ($88.10$), and the antithetic estimator ($88.14$) each cost about $0.6$ to $0.7$ $F_1$ when removed, matching the roles identified in \cref{sec:disc,sec:theory}; classifier-free guidance ($88.38$), learnable timestep aggregation ($88.42$), and the contrastive score term $\Lico$ ($88.21$) matter less on Twitter-2017. we treat this as seed noise on one dataset and keep both terms for their benefit on Twitter-2015 and in per-type behaviour. The technical appendix reports the same ablation variants on Twitter-2015 over the same seeds, where \texttt{no\_diffusion} again yields the largest drop.

\begin{table}[ht]
\caption{Evaluation-budget sweep on Twitter-2017: strict $F_1$ as a function of the number of evaluation timesteps $N$ and of the diffusion warmup length. The best value in each sweep row is in bold.}
\label{tab:nsweep}
\centering
\small
\setlength{\tabcolsep}{4pt}
\begin{tabular}{lccccc}
\toprule
\textbf{Sweep} & \multicolumn{5}{c}{\textbf{Setting}} \\
\midrule
$N$ (eval timesteps) & 1 & 2 & 5 & 10 & 20 \\
\quad strict $F_1$   & 88.66 & 88.61 & 88.47 & 88.50 & \textbf{88.71} \\
warmup epochs        & 5 & 10 & 20 & & \\
\quad strict $F_1$   & 88.19 & \textbf{88.58} & 88.44 & & \\
\bottomrule
\end{tabular}
\vspace{-1.5em}
\end{table}

The evaluation-budget sweep (\cref{tab:nsweep}) shows that strict $F_1$ is essentially flat in the number of evaluation timesteps (between $88.47$ and $88.71$ for $N$ from $1$ to $20$), so a small Monte-Carlo budget already captures most of the score signal and the verifier can be run cheaply at test time. The warmup-length sweep favours a moderate schedule, with ten warmup epochs ($88.58$) ahead of five ($88.19$); together with the ablation grid, these results trace the empirical gains back to the diffusion scorer and to the training choices analysed in \cref{sec:method,sec:disc,sec:theory}. At inference, verifier cost scales linearly with the number of candidate spans, types scored, and evaluation timesteps; antithetic evaluation doubles the denoiser forwards relative to a single-noise draw but stabilises type ranking, which is why \cref{tab:nsweep} reports similar $F_1$ for small $N$. Seed-controlled mechanism plots in the technical appendix show that gold-type score channels separate entity from non-entity spans, that mid-range noise yields the lowest reconstruction error for the correct type, and that antithetic pairing reduces score variance at fixed forward budget.

\subsection{Experimental Setup}
\label{sec:setup}

All experiments follow a fixed and reproducible protocol.

\paragraph{Baselines.}\label{sec:baselines}
We report two baseline families. \emph{Published baselines} include representative MNER methods: UMT~\citep{yu2020umt}, co-attention~\citep{zhang2018adaptive}, OCSGA~\citep{wu2020ocsga}, RpBERT~\citep{sun2021rpbert}, HVPNeT~\citep{chen2022hvpnet}, ITA~\citep{wang2022ita}, and MAF~\citep{xu2022maf}. Our control baseline, \emph{ImaginE}, is a deterministic reimplementation of imagine-and-compare verification. It shares the encoder stack, auxiliary objectives, classifier head, and regularisation with DiffImaginE, including R-Drop~\citep{liang2021rdrop} and augmentation-aware consistency, and differs only in the verifier. This controlled comparison isolates the effect of diffusion-based scoring.

\paragraph{Metrics.}\label{sec:metrics}
We report span-level strict-match precision, recall, and $F_1$, where a prediction is correct only when both its boundaries and type match the gold annotation. Per-type $F_1$ is also reported, with strict $F_1$ as the primary metric. Evaluation uses a fixed Monte-Carlo seed, making repeated evaluation of the same checkpoint deterministic. Ablations and matched comparisons use training seeds $42/43/44$ and paired tests with multiple-comparison correction. The main table reports the EMA checkpoint and threshold selected on the development set, while ablations and budget analyses report averages across seeds.

\paragraph{Implementation.}\label{sec:implementation}
The model uses RoBERTa-base and CLIP-ViT-B/32 projected to a shared width $d$. The diffusion verifier adopts a cosine schedule with $T=1000$, $M$ stratified training samples, $N$ evaluation timesteps within a selected window, classifier-free dropout $p_{\mathrm{cf}}$, guidance scale $g$, and Min-SNR clip $\gamma$. Training uses AdamW~\citep{loshchilov2019adamw}, separate learning rates for pretrained and newly introduced modules, linear warmup, AMP, EMA, and distributed data parallelism. Antithetic sampling uses paired noise $(\eps,-\eps)$: it preserves the training budget but increases evaluation from $N$ to $2N$ denoiser forwards, and can be disabled for an $N$-forward budget. Diffusion hyperparameters and decoding margins are selected on the development set and fixed for test evaluation; the full search space is provided in the technical appendix.

\section{Related Work}
\label{sec:related}

Prior MNER systems improve encoding and fusion through co-attention~\citep{moon2018multimodal,zhang2018adaptive,lu2018visual}, unified multimodal transformers with span-detection auxiliaries~\citep{yu2020umt}, object-level and relation-propagation attention~\citep{wu2020ocsga,sun2021rpbert}, and hierarchical visual prefixes with alignment objectives~\citep{chen2022hvpnet,wang2022ita,xu2022maf}. DiffImaginE is complementary: it leaves the encoder stack unchanged and modifies only the per-(span, type) verification rule.

A parallel line within MNER replaces the verifier with deterministic imagination: given a span and a type, the model predicts a single visual feature vector and scores agreement with the observation~\citep{chen2022hvpnet,xu2022maf}. This design is fast but cannot represent multimodal visual realisations of a type and offers no likelihood interpretation of the score. Our ImaginE control reimplements this paradigm with the same encoders and losses as DiffImaginE, which makes the diffusion-versus-deterministic comparison in \cref{sec:main-results,tab:ablation-results} a direct test of the verification rule rather than of the surrounding MNER stack.

Diffusion models generate data by iterative denoising~\citep{sohl2015deep,ho2020ddpm,song2019generative,song2021scoresde,nichol2021improved,kingma2021vdm} and support classifier-free guidance~\citep{dhariwal2021classifier,ho2022cfg}. A related line treats class-conditional denoising error as a zero-shot classifier~\citep{li2023diffusionclassifier,clark2023texttoimage}; likelihood-based discriminative objectives and learned timestep weightings recover much of the gap over bare ELBO scores~\citep{xian2024graphclf,jeong2025compositionality}, and antithetic noise pairing reduces estimator variance~\citep{jia2026antithetic}, an effect we formalise in \cref{prop:antithetic}. DiffImaginE applies these ideas to span-conditioned cross-modal latents in structured MNER.

\section{Conclusion}
\label{sec:conclusion}

DiffImaginE recasts multimodal NER type verification as conditional latent diffusion, scoring each type hypothesis by how well a type-conditioned denoiser reconstructs noise injected into span-localised visual evidence. We supervise the scores as classification logits, learn how to aggregate errors across timesteps, and estimate expectations antithetically; two propositions in \cref{sec:theory} justify guidance as posterior sharpening and antithetic pairing as variance reduction at fixed cost under an even/odd criterion. A matched ImaginE control attributes the observed gains to the diffusion verifier rather than to encoder or fusion changes. On Twitter-2015 and Twitter-2017 the diffusion scorer improves strict $F_1$ over this control, with ablations and budget sweeps that trace the benefit to the verifier and to the design choices in \cref{sec:disc}.

Results are limited to short-text, single-image Twitter posts, where the number of evaluation timesteps trades compute for score fidelity and where image quality varies widely. Natural next steps are multi-image or video evidence, open-vocabulary types, and applying the same span-conditioned diffusion verifier to other structured labelling tasks with visual context.

\appendix
\section*{Supplementary Material}
\addcontentsline{toc}{section}{Supplementary Material}
This supplementary material contains the proofs of the two propositions together with the full loss, hyperparameter, and ablation-configuration specifications. Equation, proposition, and assumption numbers refer to the corresponding statements in the main paper.

\section{Proof of Proposition~1 (Guided score is a tempered posterior)}
\label{app:proof-elbo}

\paragraph{Unguided score and Bayes posterior.}
Under the ELBO-consistent weighting assumption, for every type $k$ and span $s$ the expected weighted denoising error equals the negative conditional log-likelihood up to a type-independent constant,
\begin{equation}
\err_k = -\log p_\theta(\vobs \mid \zspan, \ek_k) + C,
\end{equation}
which is the diffusion-classifier identity~\citep{li2023diffusionclassifier,kingma2021vdm}. With a uniform prior $p(\ek_k \mid \zspan) = 1/K$, Bayes' rule gives
\begin{equation}
\begin{aligned}
p_\theta(\ek_k \mid \vobs, \zspan)
&= \frac{p_\theta(\vobs \mid \zspan, \ek_k)}
        {\sum_{j} p_\theta(\vobs \mid \zspan, \ek_j)} \\
&= \frac{\exp(-\err_k + C)}
        {\sum_{j}\exp(-\err_j + C)} \\
&= \softmax_k(\score_k),
\end{aligned}
\end{equation}
since $\score_k = -\err_k$ and the common factor $e^{C}$ cancels. Therefore $\argmax_k \score_k = \argmax_k p_\theta(\ek_k \mid \vobs, \zspan)$ is the Bayes-optimal type, and $\softmax_k(\score_k/\tau)$ is the temperature-$\tau$ posterior.

\paragraph{Guidance as a temperature.}
The guided score satisfies $\score_k^{(g)} = -(1 + g)\,\err_k + g\,\err_{\varnothing}$. The term $g\,\err_{\varnothing}$ is independent of $k$, so it is an additive constant inside the softmax and cancels:
\begin{equation}
\begin{aligned}
\softmax_k\!\bigl(\score_k^{(g)}/\tau\bigr)
&= \softmax_k\!\Bigl(
    \tfrac{-(1+g)\err_k + g\,\err_{\varnothing}}{\tau}
    \Bigr) \\
&= \softmax_k\!\Bigl(\tfrac{-(1+g)\err_k}{\tau}\Bigr) \\
&= \softmax_k\!\Bigl(\tfrac{\score_k}{\tau/(1+g)}\Bigr).
\end{aligned}
\end{equation}
Substituting $\score_k = \log p_\theta(\vobs\mid\zspan,\ek_k) + \text{const}$ and using the uniform prior,
\begin{equation}
\begin{aligned}
\softmax_k\!\bigl(\score_k^{(g)}/\tau\bigr)
&\propto \exp\!\Bigl(
    \tfrac{1+g}{\tau}\log p_\theta(\vobs \mid \zspan, \ek_k)
    \Bigr) \\
&\propto p_\theta(\ek_k \mid \vobs, \zspan)^{(1+g)/\tau}.
\end{aligned}
\end{equation}
For $\tau = 1$ this is $p_\theta(\ek_k\mid\vobs,\zspan)^{1+g}$, a posterior raised to the power $1 + g$. The map is a monotone (rank-preserving) sharpening: it does not change $\argmax_k$, and it makes the distribution strictly more peaked for $g > 0$. $\square$

\section{Proof of Proposition~2 (Antithetic variance reduction)}
\label{app:proof-antithetic}

We first record the even/odd decomposition, then compute the two variances, then justify the leading-order claim.

\begin{lemma}[Orthogonal even/odd decomposition]
\label{lem:evenodd}
Let $\eps \stackrel{d}{=} -\eps$ and let $f$ be square-integrable. With $f_e(\eps) = \tfrac12(f(\eps)+f(-\eps))$ and $f_o(\eps) = \tfrac12(f(\eps)-f(-\eps))$ we have $f = f_e + f_o$, $\E[f] = \E[f_e]$, $\E[f_o] = 0$, $\Cov(f_e, f_o) = 0$, and $\Var(f) = \Var(f_e) + \Var(f_o)$.
\end{lemma}
\begin{proof}
$f = f_e + f_o$ is immediate. Since $\eps \stackrel{d}{=} -\eps$, $\E[f_o] = \tfrac12(\E[f(\eps)] - \E[f(-\eps)]) = 0$, hence $\E[f] = \E[f_e]$. For the covariance, $f_e f_o$ is an odd function of $\eps$ (a product of an even and an odd function), so $\E[f_e f_o] = 0$ by symmetry, and $\Cov(f_e, f_o) = \E[f_e f_o] - \E[f_e]\E[f_o] = 0$. Orthogonality then gives $\Var(f) = \Var(f_e) + \Var(f_o)$.
\end{proof}

\paragraph{Variances of the two estimators.}
The per-sample type difference $f(\eps) = \phi_k(\eps) - \phi_j(\eps)$ is square-integrable, and both estimators use two denoiser evaluations (the antithetic estimator evaluates $\phi_k, \phi_j$ at $\eps$ and at $-\eps$; the i.i.d.\ estimator at two independent draws). Unbiasedness of both for $\E[f] = \err_k - \err_j$ follows from \cref{lem:evenodd} and linearity. For the antithetic estimator,
\begin{equation}
\begin{aligned}
\widehat{D}^{\,\mathrm{anti}}
&= \tfrac12\bigl(f(\eps) + f(-\eps)\bigr) = f_e(\eps), \\
\Var(\widehat{D}^{\,\mathrm{anti}})
&= \Var(f_e).
\end{aligned}
\end{equation}
For two independent samples $\eps^{(1)}, \eps^{(2)}$,
\begin{equation}
\begin{aligned}
\Var(\widehat{D}^{\,\mathrm{iid}})
&= \Var\!\Bigl(
    \tfrac12\bigl(f(\eps^{(1)}) + f(\eps^{(2)})\bigr)
    \Bigr) \\
&= \tfrac12\Var(f) \\
&= \tfrac12\bigl(\Var(f_e) + \Var(f_o)\bigr),
\end{aligned}
\end{equation}
using independence and \cref{lem:evenodd}. Therefore,
\begin{equation*}
\begin{aligned}
&\Var(\widehat{D}^{\,\mathrm{anti}})
    \le \Var(\widehat{D}^{\,\mathrm{iid}}) \\
&\quad\iff \Var(f_e)
    \le \tfrac12\bigl(\Var(f_e)+\Var(f_o)\bigr) \\
&\quad\iff \Var(f_e) \le \Var(f_o).
\end{aligned}
\end{equation*}

\paragraph{Leading-order dominance of the odd component.}
Write $\phi_k(\eps) = \|\epsth(\cdot, \ek_k)\|_2^2 - 2\langle \eps, \epsth(\cdot, \ek_k)\rangle + \|\eps\|_2^2$. The data-independent term $\|\eps\|_2^2$ is identical across types and cancels in $f = \phi_k - \phi_j$:
\begin{equation}
\begin{aligned}
f(\eps) ={}& \|\epsth(\cdot, \ek_k)\|_2^2
             - \|\epsth(\cdot, \ek_j)\|_2^2 \\
&{}- 2\langle \eps,\,
    \epsth(\cdot, \ek_k) - \epsth(\cdot, \ek_j)\rangle,
\end{aligned}
\end{equation}
where each $\epsth(\cdot, \ek_k) = \epsth(\sqrt{\abar_t}\vobs + \sqrt{1-\abar_t}\,\eps, \zspan, \ek_k, t)$ depends on $\eps$ through its input. Expand the denoiser to first order around $\eps = 0$, $\epsth(\cdot, \ek_k) = \bm{a}_k + \sqrt{1-\abar_t}\,J_k\,\eps + o(\|\eps\|)$, with $\bm{a}_k = \epsth(\sqrt{\abar_t}\vobs, \zspan, \ek_k, t)$ and Jacobian $J_k$. The explicit cross term contributes the odd component $-2\langle \eps, \bm{a}_k - \bm{a}_j\rangle + O(\|\eps\|^2)$. The squared-norm term is not purely even: $\|\epsth(\cdot, \ek_k)\|_2^2 = \|\bm{a}_k\|_2^2 + 2\sqrt{1-\abar_t}\,\langle J_k^{\!\top}\bm{a}_k, \eps\rangle + O(\|\eps\|^2)$, whose first-order part is also odd in $\eps$. Collecting both contributions, the odd part of $f$ is, to first order,
\begin{equation}
f_o(\eps) \approx -2\bigl\langle \eps,\; (\bm{a}_k - \bm{a}_j) - \sqrt{1-\abar_t}\,(J_k^{\!\top}\bm{a}_k - J_j^{\!\top}\bm{a}_j)\bigr\rangle,
\end{equation}
while the quadratic remainders are even to leading order. When the denoiser is well-conditioned, i.e.\ the Jacobian terms are small relative to the inter-type gap $\|\bm{a}_k - \bm{a}_j\|$, the odd part is dominated by $-2\langle \eps, \bm{a}_k - \bm{a}_j\rangle$ and $\Var(f_o) \approx 4\,\|\bm{a}_k - \bm{a}_j\|_2^2$. In this regime $\Var(f_o) > \Var(f_e)$, so by the exact criterion above the antithetic estimator reduces the variance of the type-discriminative score. We stress that the criterion $\Var(f_e) \le \Var(f_o)$ is exact; the first-order analysis only supplies a sufficient condition under which it holds. $\square$

\begin{remark}
The cancellation of $\|\eps\|_2^2$ and the oddness of the cross term are exactly why antithetic pairing helps the \emph{type comparison} more than it would help a single absolute error: the large, type-shared, even fluctuations are differenced away, leaving an odd, type-discriminative term that the antithetic average removes from the variance while preserving the mean.
\end{remark}

\section{Composite Loss Definitions}
\label{app:losses}

The total training objective is the weighted sum
\begin{align}
\mathcal{L} = {}& \Ltask + \alpha_{\mathrm{diff}}\Ldiff
    + \beta\,\Lico + \lambda_{\mathrm{clf}}\Lclf \nonumber \\
&+ \alpha_{\mathrm{rev}}\mathcal{L}_{\mathrm{ira}}^{\mathrm{rev}}
    + \beta_{\mathrm{rev}}\mathcal{L}_{\mathrm{ico}}^{\mathrm{rev}} \nonumber \\
&+ \gamma_{\mathrm{rev}}\mathcal{L}_{\mathrm{sig}}^{\mathrm{rev}}
    + \lambda_{\mathrm{bio}}\mathcal{L}_{\mathrm{bio}} \nonumber \\
&+ \lambda_{\mathrm{xm}}\mathcal{L}_{\mathrm{xmodal}}
    + \lambda_{\mathrm{gr}}\mathcal{L}_{\mathrm{ground}} \nonumber \\
&+ \lambda_{\mathrm{hard}}\mathcal{L}_{\mathrm{hard}}
    + \lambda_{\mathrm{fp}}\mathcal{L}_{\mathrm{fp}}.
\label{eq:total}
\end{align}
\Cref{tab:losses} lists each term, its role, and its default weight. Every term returns a graph-connected zero when it has no contributing spans, which keeps distributed gradient synchronisation intact. \Cref{eq:total} omits the R-Drop regulariser~\citep{liang2021rdrop} used by both DiffImaginE and the deterministic baseline, which is applied as a symmetric-KL consistency between two dropout forward passes (weight $\alpha_{\mathrm{rdrop}} = 0.5$ by default).

\begin{table}[H]
\caption{Composite loss terms. \texttt{O} denotes the non-entity class; ``entity spans'' are spans whose gold label is not \texttt{O}.}
\label{tab:losses}
\centering
\small
\setlength{\tabcolsep}{3pt}
\renewcommand{\arraystretch}{1.03}
\begin{tabular}{@{}
    >{\raggedright\arraybackslash}p{0.18\columnwidth}
    >{\raggedright\arraybackslash}p{0.42\columnwidth}
    >{\raggedright\arraybackslash}p{0.32\columnwidth}@{}}
\toprule
\textbf{Term} & \textbf{Role} & \textbf{Notes / default} \\
\midrule
$\Ltask$ & class-weighted CE on the final $K$-way logits & \texttt{O} down-weighted; online hard-\texttt{O} mining at train time \\
$\Ldiff$ & Min-SNR-weighted denoising score matching & gold type incl.\ \texttt{O}; CFG dropout $p_{\mathrm{cf}}$ \\
$\Lico$ & contrastive CE on per-type scores of entity spans & restricted to the eval timestep window \\
$\Lclf$ & likelihood-based classification on the scores & all $K$ types; learnable $\tau_{\mathrm{clf}}$ \\
$\mathcal{L}_{\mathrm{ira}}^{\mathrm{rev}}$ & reverse-imagination alignment to $\zspan$ & visual-gate weighted \\
$\mathcal{L}_{\mathrm{ico}}^{\mathrm{rev}}$ & reverse-path type contrastive & visual-gate weighted \\
$\mathcal{L}_{\mathrm{sig}}^{\mathrm{rev}}$ & VICReg-style anti-collapse on reverse imagination & prevents constant collapse \\
$\mathcal{L}_{\mathrm{bio}}$ & token-level BIO sequence labelling (shared) & masked CE on first subword \\
$\mathcal{L}_{\mathrm{xmodal}}$ & sentence-level cross-modal InfoNCE (shared) & symmetric, learnable temperature \\
$\mathcal{L}_{\mathrm{ground}}$ & attention-entropy grounding hinge (shared) & entity spans focus on fewer patches \\
$\mathcal{L}_{\mathrm{hard}}$ & hard-negative type margin ranking (shared) & gold vs.\ hardest wrong type \\
$\mathcal{L}_{\mathrm{fp}}$ & symmetric false-positive margin on \texttt{O} spans & mirror of $\mathcal{L}_{\mathrm{hard}}$; precision \\
\bottomrule
\end{tabular}
\end{table}

\clearpage
\onecolumn
\noindent
\begin{minipage}[t]{0.48\textwidth}
\section{Hyperparameters and Search Ranges}
\label{app:hyper}

\Cref{tab:hyper} lists the default hyperparameters and the ranges searched on the development set. The deterministic baseline shares every shared-stack and regularisation hyperparameter; only the verifier-specific diffusion hyperparameters are unique to DiffImaginE.

\captionsetup{hypcap=false}
\captionof{table}{Default hyperparameters and development-set search ranges.}
\label{tab:hyper}
\centering
\small
\setlength{\tabcolsep}{3pt}
\renewcommand{\arraystretch}{1.03}
\begin{tabular}{@{}
    >{\raggedright\arraybackslash}p{0.22\linewidth}
    >{\raggedright\arraybackslash}p{0.28\linewidth}
    >{\centering\arraybackslash}p{0.18\linewidth}
    >{\centering\arraybackslash}p{0.24\linewidth}@{}}
\toprule
\textbf{Group} & \textbf{Parameter} & \textbf{Default} & \textbf{Range} \\
\midrule
\multirow{3}{*}{Encoders} & text encoder & RoBERTa-base & \{base, large\} \\
                          & vision encoder & CLIP-ViT-B/32 & \{B/32, L/14\} \\
                          & shared width $d$ & 256 & \{256, 384, 512\} \\
\midrule
\multirow{4}{*}{Diffusion} & timesteps $T$ & 1000 & -- \\
                           & eval timesteps $N$ & 5 & \{1, 2, 5, 10, 20\} \\
                           & train MC steps $M$ & 4 & \{2, 4\} \\
                           & Min-SNR $\gamma$ & 5.0 & \{1, 5, $\infty$\} \\
\midrule
\multirow{3}{*}{Guidance} & CFG dropout $p_{\mathrm{cf}}$ & 0.15 & [0.0, 0.3] \\
                          & guidance scale $g$ & dev-tuned & \{1, 2, 3, 4\} \\
                          & $F_\beta$ for selection & 1.0 & \{0.5, 1.0\} \\
\midrule
\multirow{4}{*}{Warmup / loss} & warmup epochs & 10 & \{5, 10, 20\} \\
                               & $\alpha_{\mathrm{diff}}$ & 1.0 & \{0.1, 1.0, 2.0\} \\
                               & $\beta$ ($\Lico$) & 0.5 & [0.0, 1.0] \\
                               & $\lambda_{\mathrm{clf}}$ & 0.5 & \{0.1, 0.5, 1.0\} \\
\midrule
\multirow{3}{*}{Optimisation} & encoder LR & 2e-5 & -- \\
                              & new-module LR & 1e-4 & -- \\
                              & EMA decay & 0.999 & -- \\
\bottomrule
\end{tabular}
\end{minipage}
\hfill
\begin{minipage}[t]{0.48\textwidth}
\section{Ablation Configurations}
\label{app:ablation}

\Cref{tab:ablation-config} describes the core ablation variants referenced in the main paper and the hypothesis each one tests. The full grid contains 54 variants: the 15 core variants listed below plus extended variants drawn from the evaluation-budget, loss-weight, auxiliary-objective, precision-calibration, and architecture sweeps; the complete variant registry, with each variant's command-line configuration, is released with the code, and every variant is run over multiple seeds on both datasets.

\captionsetup{hypcap=false}
\captionof{table}{Core ablation variants and the hypothesis each tests.}
\label{tab:ablation-config}
\centering
\small
\setlength{\tabcolsep}{4pt}
\renewcommand{\arraystretch}{1.03}
\begin{tabular}{@{}
    >{\raggedright\arraybackslash}p{0.30\linewidth}
    >{\raggedright\arraybackslash}p{0.63\linewidth}@{}}
\toprule
\textbf{Variant} & \textbf{Hypothesis tested} \\
\midrule
main                & Reference DiffImaginE configuration. \\
no\_diffusion       & Deterministic ImaginE verifier: is the diffusion paradigm necessary? \\
no\_cfg             & Remove classifier-free guidance ($g = 0$, no dropout). \\
no\_min\_snr        & Uniform denoising weight instead of Min-SNR. \\
no\_latent\_norm    & Diffuse on the raw, unstandardised latent. \\
no\_score\_norm     & Remove the score LayerNorm + learnable temperature. \\
no\_l\_diff         & Drop the denoising score-matching term. \\
no\_l\_ico          & Drop the windowed contrastive term. \\
no\_clf\_diffusion  & Drop the likelihood-based classification objective $\Lclf$. \\
no\_warmup          & Skip the denoiser warmup phase. \\
no\_train\_on\_o    & Supervise the denoiser on entity spans only. \\
no\_learnable\_tw   & Uniform-mean timestep aggregation. \\
no\_antithetic      & I.i.d.\ Monte-Carlo noise instead of antithetic pairs. \\
greedy\_decode      & Greedy non-overlap decoding instead of the DP decoder. \\
g\_fixed\_1         & Pin $g = 1$ instead of dev-tuning it. \\
\bottomrule
\end{tabular}
\end{minipage}
\clearpage
\twocolumn

\clearpage
{\small
\bibliographystyle{plainnat}
\bibliography{references}
}

\end{document}